\documentclass{article}

\usepackage{microtype}
\usepackage{graphicx}
\usepackage{subcaption}
\usepackage{booktabs} 
\usepackage{ulem} 

\usepackage{hyperref}
\usepackage{subcaption}

 \usepackage{xcolor} 

\usepackage[preprint]{icml2026}

\usepackage{amsmath}
\usepackage{amssymb}
\usepackage{mathtools}
\usepackage{amsthm}
\usepackage{ulem}
\usepackage{comment}

\usepackage[capitalize,noabbrev]{cleveref}

\theoremstyle{plain}
\newtheorem{theorem}{Theorem}[section]

\newtheorem{lemma}[theorem]{Lemma}

\theoremstyle{definition}

\theoremstyle{remark}

\usepackage[textsize=tiny]{todonotes}

\icmltitlerunning{Reliability-Aware ProbDPP for Robust Data Selection}

\begin{document}

\twocolumn[
  \icmltitle{Reliability-Aware Determinantal Point Processes for Robust Informative Data Selection in Large Language Models}



  \icmlsetsymbol{equal}{*}

\begin{icmlauthorlist}
  \icmlauthor{Ahmad Sarlak}{clemson}
  \icmlauthor{Abolfazl Razi}{other}
\end{icmlauthorlist}

\icmlaffiliation{clemson}{Clemson University, Clemson, SC, USA}
\icmlaffiliation{other}{School of Computing, Clemson University, Clemson, SC, USA}

\icmlcorrespondingauthor{Ahmad Sarlak}{asarlak@g.clemson.edu}

  \icmlkeywords{Machine Learning, ICML}

  \vskip 0.3in
]



\printAffiliationsAndNotice{}  

\begin{abstract}


Informative data selection is a key requirement for large language models (LLMs) to minimize the amount of data required for fine-tuning, network distillation, and token pruning, enabling fast and efficient deployment, especially under computational and communication constraints.
Traditional subset selection methods, including those based on Determinantal Point Processes (DPP), focus on maximizing diversity but assume that selected data batches are always available error-free. This presumption prohibits their use under partial storage outage, imperfect communication, and stochastic access failures. Furthermore, we show that the original formulation collapses under such conditions.
To address this gap, we introduce ProbDPP, a novel reliability-aware implementation of k-DPP that accounts for probabilistic data access by recasting the objective function with a regularization term that remains well-posed and decomposes into a geometric diversity term and unreliability cost. The resulting objective facilitates robust selection of diverse data batches under uncertainty. Furthermore, we frame this reliability-aware diversity maximization as a combinatorial semi-bandit problem and propose a UCB-style algorithm to efficiently learn the unknown reliability online. Theoretical analysis provides regret bounds for the proposed approach, ensuring performance guarantees.
  
\end{abstract}

\section{Introduction}
LLMs are increasingly deployed across a wide range of applications. They rely on large amounts of data at training and fine-tuning phase (curated data for fine-tuning or distillation). Recent work highlights that careful data selection and curation, including high-quality instruction data mining and efficient subset selection for fine-tuning, can dramatically improve LLM performance and efficiency, often matching or surpassing full-data baselines while using only a fraction of the original dataset. \cite{wangnice, jiang2025importance, zhang2025survey}. 

Also, the speedy inference of LLMs is severely constrained by the input prompt. 
In many systems, prompts are constructed by aggregating information from heterogeneous upstream sources, including retrieval systems, databases, tools, documents, and other models \cite{sun2025docs2kg}. Upstream components often produce a large pool of candidate units (retrieved passages, tool outputs, or model-generated chunks), which cannot all be included due to strict context-window and compute/storage budgets, underscoring the significance of selective data use. One notable example is autonomous driving with cooperative perception (CAP), where sensors, perception modules, and auxiliary models generate multiple candidate signals or summaries \cite{chiu2025v2v, wang2024reasoning, zhang2024efficient}. Selecting an informative and robust subset of these candidates for real-time operation of multi-modal LLMs remains a central systems challenge. 

This selection is often performed in a black-box regime, where the LLM’s internal parameters and activations are inaccessible and only input-side pre-processing is permitted \cite{nagle2024fundamental, li2023compressing, jiang2023llmlingua, jiang2024longllmlingua}, making principled subset selection a central systems bottleneck \cite{rahman2026dtc}.
Another key challenge in informative unit selection -- both for inference-time prompt construction and for training-time data selection, such as fine-tuning or distillation -- is that candidate sources are rarely reliable \cite{sarlak2023diversity}. Candidate units may be corrupted or dropped due to communication errors (e.g., in cooperative or edge-based systems), lost due to storage or caching failures, or degraded by the imperfect operation of automated tools, including retrieval systems, generative models, or data augmentation pipelines. As a result, extracted passages may be off-topic, tool outputs may fail or be malformed, and upstream model generations may be noisy or misleading due to hallucinations or domain shift \cite{alinejad2024evaluating, cho2024typos, han2024rag}. Therefore, selected data units may fail to make meaningful contributions to the final prompt or training process \cite{sarlak2025extended}. Under a strict budget, such stochastic “dropouts” can be catastrophic: a method may allocate capacity to unreliable units while excluding others that are both informative and reliable \cite{ye2025fit, huang2025prunevid}.


Diversity-driven objectives offer a natural way to mitigate redundancy on a tight budget \cite{chen2025rd, lamont20243d}. In particular, log-determinant (log-det) diversity, a fundamental concept used in DPP-based methods, encourages selecting embedding vectors that span a large volume, providing broad semantic coverage while discouraging near-duplicates \cite{hough2006determinantal, taskar2013determinantal}. Log-det/DPP objectives have therefore become a commonplace solution for subset selection and have influenced token- and sentence-level prompt compression \cite{wang2024effective}. However, the classical log-det formulation assumes that selected units are deterministically present and usable. This assumption breaks in multi-source prompting pipelines, where each selected unit may be irrelevant, corrupted, or missing with nonzero probability \cite{kulesza2010structured}.

This paper shows that the most direct way of incorporating unreliability is ill-posed. 
Specifically, we model stochastic source unreliability using a binary masking vector, where each candidate unit is independently dropped with an unknown reliability probability. 
Under this model, the natural objective is to maximize the expected log-determinant of the selected kernel, which captures diversity while accounting for random dropouts. However, directly optimizing this expectation leads to an ill-posed objective.
When unit availability is modeled by independent Bernoulli variables, the expected log-determinant diverges to $-\infty$ if
any selected unit has a nonzero failure probability. In other words, naive “expected log-det under dropouts” collapses and cannot serve as a meaningful objective for reliability-aware subset selection. This exposes a fundamental gap between the widely used diversity-based selection principles \cite{neumann2021diversifying, jha2024characterizing} and the stochastic reality of modern prompting pipelines \cite{cuconasu2024power}.

To address this issue, we propose \textbf{ProbDPP}, by adding a regularization term to the randomly masked kernel that restores well-posedness and yields a tractable, reliability-aware selection objective. The resulting expected diversity is finite and admits an exact decomposition into (i) a geometric log-determinant term that captures embedding-space diversity and (ii) an additive per-unit reliability term that explicitly favors dependable sources. This decomposition preserves the combinatorial structure that makes log-determinant objectives attractive, while explicitly incorporating stochastic availability. This enables diverse data selection under probabilistic data access with known success probabilities for each data unit. 

Furthermore, as a practical implementation of our reliability-aware objective when source reliabilities are unknown and must be learned online, we cast subset selection as a combinatorial semi-bandit problem with semi-bandit feedback, where per-unit success or failure is inferred gradually from the received data \cite{kveton2015tight, wen2015efficient}. Within this framework, we develop an efficient UCB-style algorithm \cite{garivier2011kl} to gradually select data units that balance diversity and reliability. We provide theoretical guarantees in the form of regret upper bounds, together with a complementary lower bound that characterizes the intrinsic difficulty of the problem \cite{lattimore2020bandit}.

\paragraph{Contributions.}
Our contributions are summarized below.
\begin{enumerate}
    \item We prove an ill-posedness result by showing that the naive expected log-det diversity under Bernoulli dropouts diverges to $-\infty$ whenever at least one selected data unit is subject to random dropout with nonzero probability.
    
    \item We develop a probabilistic formulation of log-determinant subset selection for stochastic availability scenarios, and propose a principled regularization that restores well-posedness, yielding a finite expected objective with an exact separation between embedding-space diversity and source reliability, for probabilistic data access with known success probabilities.
    
    \item We formulate reliability-aware diversity maximization as a combinatorial semi-bandit for probabilistic dropouts with unknown, learnable reliabilities and propose an efficient UCB-style algorithm for the online setting.
    
    \item We establish regret upper bounds and a corresponding information-theoretic lower bound, characterizing the minimal number of observations required to jointly learn unit reliabilities and perform diversity-aware subset selection under stochastic availability.
\end{enumerate}

\section{Related Work}

Selecting a small and useful subset of context items under a strict budget is a central problem in both in-context learning (ICL) and retrieval-augmented generation (RAG). Simply selecting the top-ranked passages or demonstrations often may lead to redundant content with minimal synergistic gains. 
To mitigate these issues, several recent methods use diversity-aware subset selection, including submodular objectives and determinantal point processes (DPPs), to balance relevance, coverage, and redundancy. For instance, LM-DPP casts demonstration selection as a selective annotation problem in ICL and uses a language-model-based DPP to favor examples that are both diverse and low in uncertainty \cite{wang2024lmdpp}. Similarly, SMART-RAG applies a DPP-style objective to select retrieved passages that are relevant, diverse, and non-conflicting, improving unsupervised context selection \cite{li2024smartrag}. While these approaches show that log-det/DPP objectives are effective for budgeted context selection, they assume that selected context items are always available and usable. 
We show that the original log-det/DPP formulation falls apart for being ill-posed under probabilistic data availability (e.g., data drop due to storage or communication issues in LLM-powered cooperative driving scenarios). This underscores the need for a reliability-aware reformulation that remains tractable, representing geometric benefits of DPP-based selection for stochastic data access. 

A closely related line of work studies prompt compression as an input-level optimization for black-box LLMs, focusing on token-level compressors that reduce inference cost and free up context budget without accessing model internals \cite{jiang2023llmlingua}. Recent work further formalizes prompt compression through a rate--distortion framework, characterizing the optimal trade-off between retained information and prompt length via a linear program and its dual \cite{jiang2024longllmlingua}. This distortion rate curve provides a theoretical baseline for assessing pruning and compression methods, revealing a significant performance gap between existing approaches and the theoretical limits, and highlighting the importance of conditioning compression on the downstream query \cite{nagle2024fundamental}. While this framework provides a principled notion of optimal compression under a budget, it models compression as a deterministic operation and does not capture stochastic unreliability in upstream units, such as corrupted or distorted chunks produced by imperfect retrieval or generative tools, packet drops over noisy communication links, or missing items due to storage/caching failures.

Our work is complementary and offers a transformative perspective on this problem. Inspired by the issues we observed when developing LLM-powered cooperative driving for autonomous vehicles, we focus on reliability-aware subset selection under an explicit stochastic failure model by developing a regularized objective with an online semi-bandit learning approach for unknown reliabilities, noting that naive expected log-det collapses under Bernoulli dropouts. 

The core of our system is based on DPP, which arises as a standard probabilistic model for diversity-aware subset selection: given a positive semidefinite kernel, they favor subsets with large determinants, which corresponds to selecting feature vectors that span a large volume in embedding space.
This log-determinant objective favors diverse, non-redundant sets and admits efficient sampling and approximate MAP inference (e.g., $k$-DPP and greedy methods), motivating widespread use in retrieval, summarization, and experimental design \cite{kulesza2011learning, kulesza2011k, kulesza2010structured, taskar2013determinantal}. 
Unlike classical DPP settings that assume deterministic item availability, our formulation accommodates stochastic dropouts.

Robustness to uncertainty, failures, and deletions has been studied extensively in the submodular maximization literature. Robust submodular maximization typically optimizes performance under perturbations or worst-case uncertainty sets (e.g., maximizing the minimum value across multiple objective functions), with applications such as sensor placement and bi-criteria guarantees under structured constraints \cite{anari2019structured}. Deletion-robust submodular maximization similarly studies worst-case removal of selected elements after selection, motivated by dynamic streams and the ``right to be forgotten'' \cite{mirzasoleiman2017deletion}. Related robustness ideas also appear in sensor placement and experimental design under node failures or model uncertainty \cite{krause2008near}. 
These lines of work differ from our setting in two ways. First, they predominantly model worst-case deletions or bounded perturbations, whereas we consider stochastic independent per-unit failures that capture both missing items (e.g., communication/storage dropouts) and effectively useless items (e.g., corrupted or distorted chunks produced by imperfect retrieval or generative tools). Second, our objective involves the expected log-determinant under Bernoulli masking, for which singularity events occur with nonzero probability; consequently, the naive expected log-det can diverge to $-\infty$, a pathology not addressed by worst-case deletion models. Our contribution is therefore complementary: we identify this failure mode and introduce a minimal regularization that restores well-posedness and yields a tractable decomposition suitable for both offline optimization and online learning.

Our online optimization problem fits the stochastic combinatorial semi-bandit framework: at each round, the learner selects a subset of items and observes per-item stochastic outcomes (semi-bandit feedback). UCB-style optimism is a standard approach in this setting \cite{sankararaman2016semi}. 
For linear reward models with independent item weights, algorithms such as CombUCB1 achieve near-optimal regret bounds when an offline optimization oracle is available \cite{kveton2015tight}. For Bernoulli or bounded rewards, KL-UCB-style confidence bounds are known to be optimal, matching Lai-Robbins lower bounds in the Bernoulli case \cite{garivier2011kl}. Our problem differs from these standard settings because the per-round reward decomposes into a known nonlinear log-determinant diversity term and an unknown additive reliability-dependent term learned from Bernoulli outcomes. This results in a structured combinatorial semi-bandit problem that combines semi-bandit learning with a DPP/log-det optimization oracle.

Lastly, it is noteworthy that client or participant selection is studied in different contexts. For instance, it constitutes an important component of federated and edge learning systems, where communication constraints limit the number of participants, and non-IID data across clients can slow down the convergence. Prior work has proposed diversity-aware client selection methods to improve statistical coverage, including DPP-based approaches such as federated learning with DPP-based participant selection, which use k-DPP sampling over profiled client data to speed up convergence and reduce communication overhead \cite{zhang2023dpp, bastola2024fedmil}. While these methods demonstrate the value of DPPs for promoting statistical diversity in federated learning, they typically assume that client availability and participation quality are priorly known or fixed. In contrast, our work focuses on diversity-aware selection under stochastic availability and integrates log-det/DPP selection with semi-bandit learning to estimate unknown and even time-varying reliabilities online, addressing a setting that is largely orthogonal to existing federated client selection methods.

\section{Problem Formulation}
We consider $N$ candidate sources (clients) indexed by $[N]=\{1,\dots,N\}$. 
At each round $t=1,2,\dots,T$, the server selects a fixed-size subset
\begin{align}
S_t \subseteq [N], \qquad |S_t| = K. 
\end{align}
Due to imperfect communication and stochastic unavailability, each selected source $i\in S_t$ independently succeeds with unknown probability $\alpha_i\in[0,1]$. 
Let $z_{i,t}\sim\mathrm{Bernoulli}(\alpha_i)$ denote the success indicator; the learner observes $\{z_{i,t}: i\in S_t\}$ only after attempting to fetch data (semi-bandit feedback).

When source $i$ succeeds, the server obtains an update from which it extracts a unit-norm feature embedding $f_i\in\mathbb{R}^d$ with $\|f_i\|_2=1$ (equivalently, the received feature at time $t$ is $z_{i,t}f_i$). 
Similarity across received data embeddings is modeled by the positive semidefinite Gram matrix
\begin{align}
G\in\mathbb{R}^{N\times N}, \qquad G_{ij}=\langle f_i,f_j\rangle, \qquad G\succeq 0. 
\end{align}
A standard diversity score for a selected subset $L\subseteq[N]$ is
\begin{align}
g(L)\;\triangleq\;\log\det(G_{L,L}). 
\end{align}

\paragraph{Why the naïve dropout-aware log-det collapses.}
Under stochastic availability, the actually received kernel for a chosen subset $L$ is masked by the dropout indicators. 
Let $z_t=(z_{1,t},\dots,z_{N,t})$ be the binary mask vector and define the diagonal mask
\begin{align}
M_L(z_t)\;=\;\mathrm{diag}(z_{i,t}: i\in L). 
\end{align}
The unregularized received kernel is
\begin{align}
\Sigma_{L,L}(z_t)\;=\;M_L(z_t)\,G_{L,L}\,M_L(z_t). 
\end{align}
Whenever any selected item drops ($z_{i,t}=0$), $\Sigma_{L,L}(z_t)$ becomes singular and $\log\det(\Sigma_{L,L}(z_t))=-\infty$.

\begin{lemma}[Ill-posedness under dropouts]
\label{lem:illposed}
Fix any subset $L\subseteq[N]$ such that there exists $i\in L$ with $\alpha_i<1$. 
Let $z_i\sim\mathrm{Bernoulli}(\alpha_i)$ be independent and define $\Sigma_{L,L}(z)=M_L(z)\,G_{L,L}\,M_L(z)$ with $M_L(z)=\mathrm{diag}(z_i:i\in L)$. 
Then
\begin{align}
    \mathbb{E}_z\big[\log\det(\Sigma_{L,L}(z))\big] = -\infty.
\end{align}

\end{lemma}
\noindent\emph{Proof.} Deferred to \ref{app:illposed}. 

\subsection{Reliability-aware ProbDPP Objective}
To obtain a finite and meaningful objective under dropouts, we adopt a regularized ProbDPP formulation.
For $\varepsilon>0$, we define the regularized masked kernel as
\begin{align}
&W_L(z_t) \;=\; M_L(z_t)+\varepsilon I, \\
&\Sigma^{(\varepsilon)}_{L,L}(z_t) \;=\; W_L(z_t)\,G_{L,L}\,W_L(z_t). 
\end{align}
The resulting regularized expected diversity is
\begin{align}
D_\varepsilon(L)\;\triangleq\;\mathbb{E}_{z_t}\!\left[\log\det\!\Big(\Sigma^{(\varepsilon)}_{L,L}(z_t)\Big)\right]. 
\end{align}

\begin{lemma}[Exact decomposition of the regularized objective]
\label{lem:decomp}
Let $\varepsilon>0$ and $\Sigma^{(\varepsilon)}_{L,L}(z)=W_L(z)\,G_{L,L}\,W_L(z)$ with $W_L(z)=M_L(z)+\varepsilon I$. 
Then
\begin{align}
D_\varepsilon(L) \;=\; \log\det(G_{L,L}) + \sum_{i\in L} r_i(\alpha_i,\varepsilon), 
\end{align}
where the per-item reliability term is
\begin{align} \label{eq:ri}
r_i(\alpha,\varepsilon) \;=\; 2\Big[\alpha\log(1+\varepsilon) + (1-\alpha)\log\varepsilon\Big]. 
\end{align}
\end{lemma}
\noindent\emph{Proof.} Deferred to \ref{app:decomposition}.

\paragraph{Interpretation.}
Lemma~\ref{lem:decomp} separates geometric diversity $\log\det(G_{L,L})$ from an additive reliability term $\sum_{i\in L} r_i(\alpha_i,\varepsilon)$, preventing objective collapse under stochastic availability.

\paragraph{Online objective.}
Since the reliabilities $\{\alpha_i\}_{i=1}^N$ are unknown, the learner must select $\{S_t\}_{t=1}^T$ using only semi-bandit feedback $\{z_{i,t}: i\in S_t\}$. 
The goal is to maximize cumulative reliability-aware diversity via selecting $S_t$ at time points $t=1,2,...T$:
\begin{align}
\max_{\substack{|S_t|=K\\ S_t\subseteq[N]}}
\;\sum_{t=1}^T D_\varepsilon(S_t).
\end{align}

\subsection{RA-kDPP: A Combinatorial Semi-Bandit Formulation}

The reliability-aware client selection problem naturally fits the combinatorial semi-bandit setting: at each round, the learner selects
multiple items under a cardinality constraint, and observes the successful response 
only for the selected items. In this section, we formalize the problem as a combinatorial semi-bandit with semi-bandit feedback, explicitly absorbing the unreliability parameters into the objective function. 

To this end, we model the bandit actions with a binary vector $A_t \in \{0,1\}^N$, $\qquad \|A_t\|_1 = K$ 
where $A_{t,i} = 1$ if and only if client $i$ is selected and $A_{t,i} = 0$ otherwise.

Specifically, we define the induced linear coefficient
\begin{align}
    \theta_i \;\triangleq\; r_i(\alpha_i,\varepsilon),
\end{align}
where $r_i(\cdot,\varepsilon)$ is a known reliability-to-reward mapping
parameterized by $\varepsilon$, as defined in \ref{eq:ri}.
The unknown parameter vector is thus
\begin{align}
    \theta = (\theta_1,\dots,\theta_N) \in \mathbb{R}^N 
\end{align}

At each round $t$, the learner receives a linear reward contribution
\begin{align}
    \langle A_t, \theta \rangle
=
\sum_{i=1}^N A_{t,i}\,\theta_i
=
\sum_{i=1}^N A_{t,i}\, r_i(\alpha_i,\varepsilon).
\label{Attheta}
\end{align}

In addition to this unknown linear term, each action $A_t$ is associated with a
known deterministic combinatorial component
\begin{align}
    g(A_t)
=
\log \det \bigl(G_{S_t,S_t}\bigr),
\label{gat}
\end{align}
where $G$ is a known Gram matrix and $G_{S_t,S_t}$ denotes the principal submatrix
indexed by the selected subset $S_t$.

Combining both components, the per-round objective is to maximize 
\begin{equation}
f(A_t)
=
g(A_t) + \langle A_t, \theta \rangle .
\label{eq:linear_plus_known}
\end{equation}

This formulation shows that the reliability-aware client selection problem
induced by ProbDPP can be viewed as a combinatorial semi-bandit with a
known nonlinear term and an unknown linear term.
The learner estimates the Bernoulli parameters $\alpha_i$ from semi-bandit
feedback and applies optimism to the induced coefficients $\theta_i$ when
selecting actions.

\subsection{Semi-bandit feedback and estimation}
After selecting an action $A_t \in \mathcal{A}$, the learner observes
semi-bandit feedback: Let
\begin{align}
n_i(t)
&= \sum_{s=1}^t \mathbb{I}\{A_{s,i}=1\}
\end{align}
denote the number of times coordinate $i$ has been selected up to round $t$.
The empirical mean estimator of the unknown reliability $\alpha_i$ is
\begin{align}
\hat{\alpha}_i(t)= \frac{\sum_{s=1}^t \mathbb{I}\{A_{s,i}=1\} \, z_{i,s}}{\sum_{s=1}^t \mathbb{I}\{A_{s,i}=1\} }
\qquad \text{for } n_i(t) \ge 1 .
\end{align}

This estimator coincides with the standard semi-bandit estimator for
combinatorial bandits with deterministic action selection, when success estimates $\hat{\alpha}_i$ approach their target values $\alpha_i$.
No importance weighting or unbiased correction is required, since
each observed outcome $z_{i,t}$ corresponds to an actual pull of
coordinate $i$.

\subsection{Action selection with KL-UCB optimism}

To balance exploration and exploitation, we construct optimistic upper
confidence bounds for the unknown client reliabilities using the
KL-UCB principle in \citet{lattimore2020bandit}.
Since the observed feedback variables $z_{i,t}$ are Bernoulli, an upper
confidence bound for each reliability parameter $\alpha_i$ is defined as

\begin{align}
U_i^\alpha(t) =
&\max \Bigl \{
u \in [0,1]: \\  
\nonumber
& 
n_i(t)\,
\mathrm{kl}\!\bigl(\hat P_{\alpha_i(t)} \,\|\, P_u\bigr)
&\le
\log t
 + c \log\log t
\Bigr\},
\label{eq:klucb_alpha}
\end{align}

where $P_u$ represents Bernoulli distribution with parameter $u$, 
and
$\mathrm{kl}(P\|Q)$ denotes the Kullback-Leibler divergence
between distributions $P$ and $Q$.
The constant $c>0$ controls the confidence level.

The optimistic reliability bounds are then mapped into optimistic linear
contributions through the known reliability-to-reward transformation.
Specifically, we define the optimistic coefficient
\begin{align}
    \tilde \theta_i(t)
=
r_i\!\bigl(U_i^\alpha(t), \varepsilon\bigr),
\end{align}
and collect these into the optimistic weight vector
\begin{align}
    \tilde \theta(t)
=
\bigl(\tilde \theta_1(t),\dots,\tilde \theta_N(t)\bigr).
\end{align}

At each round $t$, the learner selects the action
\begin{equation}
A_t
\in
\arg\max_{a \in \mathcal{A}}
\Bigl\{
g(a) + \langle a, \tilde \theta(t) \rangle
\Bigr\},
\label{eq:oracle_klucb}
\end{equation}
where $g(a)$ and $\langle a, \tilde \theta(t) \rangle$ is defined respectively in ~\eqref{gat} and ~\eqref{Attheta}
which corresponds to reliability-aware ProbDPP selection with
coordinate-wise KL-UCB optimism applied to the induced linear
coefficients.
The optimization problem~\eqref{eq:oracle_klucb} is solved using the same
oracle or greedy procedure as in the non-bandit setting. The summary of this process is presented in \textbf{Algorithm 1}.

\begin{algorithm}[tb]
  \caption{ProbDPP}
  \label{alg:probDPP_klucb}
  \begin{algorithmic}
    \STATE {\bfseries Input:} horizon $T$
    \STATE {\bfseries Action set:} $\mathcal{A}=\{a\in\{0,1\}^N:\|a\|_1=K\}$
    \STATE Initialize $n_i(0)=0$, $\hat{\alpha}_i(0)=0$ for all $i\in\{1,\dots,N\}$
    \FOR{$t=1$ {\bfseries to} $T$}
      \FOR{$i=1$ {\bfseries to} $N$}
        \IF{$n_i(t-1)\ge 1$}
          \STATE Compute $U_i^\alpha(t)$:
          \STATE \hspace{1em}$
          \begin{aligned}
          U_i^\alpha(t) = \max_{u\in[0,1]}\; & u \\
          \text{s.t.}\quad
          n_i(t\!-\!1)\,\mathrm{kl}\!\bigl(\hat{\alpha}_i(t\!-\!1)\,\|\,u\bigr)
          &\le \log t \\
          &\quad + c\log\log t
          \end{aligned}
          $
        \ELSE
          \STATE Set $U_i^\alpha(t)=1$
        \ENDIF
        \STATE Set $\tilde{r}_i(t)=r_i\!\bigl(U_i^\alpha(t),\varepsilon\bigr)$
      \ENDFOR

      \STATE Select $A_t$:
      \STATE \hspace{1em}$
      \begin{aligned}
      A_t \in \arg\max_{a\in\mathcal{A}}\; \{\, & g(a) \\
      &+ \langle a,\tilde{r}(t)\rangle \,\}
      \end{aligned}
      $
      \STATE Observe $\{z_{i,t}: A_{t,i}=1\}$
      \FOR{each $i$ with $A_{t,i}=1$}
        \STATE Update $n_i(t)$ and $\hat{\alpha}_i(t)$
      \ENDFOR
    \ENDFOR
  \end{algorithmic}
\end{algorithm}

\section{Performance Analysis}

In this section, we establish regret guarantees for the proposed online learning framework.
All proofs are deferred to the appendix unless stated otherwise.

\subsection{Benchmark and Regret}

We define the best fixed action in hindsight as
\begin{equation}
A^\star \in \arg\max_{a\in\mathcal{A}} f(a),
\end{equation}
i.e., the size-$K$ subset maximizing the objective under the true reliabilities.
The cumulative pseudo-regret is
\begin{equation}
\mathrm{Reg}_T \triangleq \sum_{t=1}^T \big(f(A^\star)-f(A_t)\big),
\end{equation}
where $A_t$ is the action (i.e. data batches selected) by Algorithm~\ref{alg:probDPP_klucb} at time t.

To express the regret bound in item-wise terms, define for each item $i\in[N]$ the smallest strictly positive gap among suboptimal actions that include $i$:
\begin{equation}
\Delta_i \triangleq 
\min\Big\{ f(A^\star)-f(a)\;:\; a\in\mathcal{A},\ a_i=1,\ f(a)<f(A^\star)\Big\},
\end{equation}
with the convention $\Delta_i=+\infty$ if no strictly suboptimal action contains item $i$.

\begin{theorem}[Regret of ProbDPP]\label{thm:upper}
Assume semi-bandit feedback with independent Bernoulli outcomes: whenever item $i$ is selected at round $t$, we observe $z_{i,t}$ for an unknown reliability $\alpha_i$.
Let $\{A_t\}_{t=1}^T$ be produced by Algorithm~\ref{alg:probDPP_klucb}. Define
\begin{equation}
\beta_\varepsilon \triangleq 2\log\!\Big(\frac{1+\varepsilon}{\varepsilon}\Big),
\qquad
\alpha_i^+ \triangleq \min\Big\{1,\ \alpha_i + \frac{\Delta_i}{K\beta_\varepsilon}\Big\}.
\end{equation}
Then
\begin{equation}
\mathbb{E}[\mathrm{Reg}_T]
\;\le\;
\Delta_{\max}\sum_{i=1}^N
\frac{\log T + c\log\log T}
{D_{\mathrm{KL}}\!\big(P_{\alpha_i}\,\|\,P_{\alpha_i^+}\big)}
\;+\; O(1),
\label{regt}
\end{equation}
where $\Delta_{\max}\triangleq \max_{a\in\mathcal{A}}(f(A^\star)-f(a))<\infty$,
$c>0$ is the KL-UCB exploration constant, and
\begin{align}
D_{\mathrm{KL}}\!\big(P_{\alpha_i}\,\|\,P_{\alpha_i^+}\big)
=
\alpha_i \log\frac{\alpha_i}{\alpha_i^+}
+
(1-\alpha_i)\log\frac{1-\alpha_i}{1-\alpha_i^+}.
\end{align}
\end{theorem}

Proof idea (see Appendix \ref{app:proof_thm53} for details).
On the KL-UCB high-probability event, the optimistic per-item rewards upper bound the true rewards.
If $A_t$ is suboptimal, the gap $f(A^\star)-f(A_t)$ can be charged to at least one selected item $i$ whose optimism exceeds $\Delta_i/K$.
Since $r_i(\alpha,\varepsilon)$ is affine in $\alpha$, this implies $U_i^\alpha(t-1)\ge \alpha_i^+$, and standard KL-UCB counting bounds the expected number of such rounds by $(\log T + c\log\log T)/\mathrm{kl}(\alpha_i,\alpha_i^+)$.
Summing over items yields \eqref{regt}.

\subsection{Lower Bound for ProbDPP}

We show that logarithmic regret is unavoidable in the reliability-aware
ProbDPP framework by reduction to a classical Bernoulli bandit.
A complete proof with sharp constants is given in Appendix \ref{app:proof_lower}.

\paragraph{Restricted instance.}
Consider the special case $K=1$ and choose the Gram matrix $G$ such that
$g(a)$ is constant for all singleton sets $S=\{i\}$.
Then $g(a)$ does not affect regret differences, and each action corresponds
to selecting a single item $i\in[N]$.
The per-round objective reduces to
\begin{align}
f(i) = r_i(\alpha_i,\varepsilon),
\end{align}
and the problem becomes an $N$-armed stochastic bandit with $z_{i,t}$ and
$\mu_i \triangleq r_i(\alpha_i,\varepsilon)$.

Let $i^\star=\arg\max_i \mu_i$ and define the reward gaps
\begin{align}
\Delta_i \triangleq \mu_{i^\star}-\mu_i > 0,
\qquad i\neq i^\star.
\end{align}
Since $r_i(\alpha,\varepsilon)$ is affine in $\alpha$, these gaps are
equivalently proportional to reliability gaps
$\alpha_{i^\star}-\alpha_i$.

\paragraph{Information-theoretic argument.}
Fix any policy $\pi$ and a suboptimal arm $i\neq i^\star$.
Construct an alternative instance $\alpha'$ that differs from $\alpha$
only in coordinate $i$, with $\alpha'_i>\alpha_{i^\star}$, so that arm $i$
is uniquely optimal.
Let $P_\alpha$ and $P_{\alpha'}$ denote the distributions over the
interaction history induced by $\pi$ under $\alpha$ and $\alpha'$.
Since the two instances differ only on arm $i$,
\begin{align}
D_{\mathrm{KL}}(P_\alpha\|P_{\alpha'})
=
\mathbb{E}_\alpha[n_i(T)]\,\mathrm{kl}(\alpha_i\|\alpha'_i),
\end{align}
where $n_i(T)$ is the number of times arm $i$ is selected up to time $T$.

\paragraph{Consequence for regret.}
Standard change-of-measure arguments imply that any policy achieving
subpolynomial regret must satisfy
\begin{align}
\mathbb{E}_\alpha[n_i(T)] \;\gtrsim\; \log T
\quad \text{for all } i\neq i^\star,
\end{align}
and therefore
\begin{align}
\mathbb{E}_\alpha[\mathrm{Reg}_T]
\;\gtrsim\;
\sum_{i\neq i^\star}\Delta_i \log T
\;\ge\;
(N-1)\,\delta_{\min}\,\log T,
\end{align}
where $\delta_{\min}\triangleq\min_{i\neq i^\star}\Delta_i$.

\paragraph{Implication for ProbDPP.}
Since this Bernoulli bandit is a special case of the reliability-aware
ProbDPP semi-bandit formulation, no algorithm can achieve $o(\log T)$
regret uniformly over instances.
Combined with Theorem~4.1, this shows that ProbDPP achieves
order-optimal regret in $T$ (up to constants and gap-dependent factors).

\section{Simulation and Experiment Result}
\paragraph{Robustness Under Stochastic Context Unavailability.}

We evaluate robustness under stochastic context unavailability in a long-context question answering setting, MeetingBank dataset \cite{hu2023meetingbank}, using a local llm "llama3". First, we compress (select) the data, then give the compressed (selected) data to the llm, where only a limited number of retrieved chunks can be included. The selected chunks are then dropped with a nonzero probability at inference time, reflecting realistic retrieval-augmented and multi-source prompting pipelines. We consider a fixed context budget (up to 30 chunks, each constituting 60–90 tokens depends on the context length). The dropouts follow an independent Bernoulli distribution with rate $\alpha \in [0,\,0.8]$, while keeping the downstream LLM and decoding parameters fixed across methods. We compare \textsc{ProbDPP} against commonly used baselines for tight budget regimes: Random selection as a lower bound, and LLMLingua2 \cite{pan2024llmlingua} / LLMLingua-small \cite{jiang2023llmlingua}, which apply state-of-the-art prompt compression but do not model reliability. Table~\ref{tab:comparison_results1} reports results on the MeetingBank dataset (120 meetings, 120 questions) under stochastic chunk unavailability. Random selection performs worst across all metrics, while LLMLingua-based compression yields moderate improvements by reducing redundancy, but remains limited due to the lack of reliability modeling. In contrast, \textsc{ProbDPP} achieves the best performance in all metrics, including Token-F1 (31.3), ROUGE-L (31.2), BERTScore (36.9), and Exact Match (19.4), with relative gains of 36\% over the strongest baseline (LLMLingua2). These gains are especially pronounced for BERTScore and Exact Match, indicating improved semantic alignment and more reliable exact answers under noisy inputs. All methods use the same context budget and are evaluated under identical $\alpha$ conditions. Therefore, the improvements are derived from \textsc{ProbDPP}’ reliability-aware diversity objective, which balances coverage with the probability that selected segments remain available at inference time.

\begin{table}[t]
  \centering
  \small
  \setlength{\tabcolsep}{4pt}
  \renewcommand{\arraystretch}{1.0}
  \begin{tabular}{c|cccc}
    \toprule
    & Llmlingua2 & Llmlingua-small & Random & ProbDPP \\
    \midrule
    Token-F1   & 28.8 & 29.5 & 23.6 & \textbf{31.3} \\
    ROUGE-L    & 28.8 & 29.4 & 23.8 & \textbf{31.2} \\
    BERTScore  & 32.5 & 30.6 & 27.0 & \textbf{36.9} \\
    EM         & 17.3 & 16.9 & 15.0 & \textbf{19.4} \\
    \bottomrule
  \end{tabular}
  \caption{Robustness under stochastic context unavailability on MeetingBank (max 30 chunks). Higher is better; best in bold.}
  \label{tab:comparison_results1}
\end{table}
\vspace{0.8em}

\begin{table}[t]
  \centering
  \small
  \setlength{\tabcolsep}{4pt}
  \renewcommand{\arraystretch}{1.05}
  \begin{tabular}{c|cccc}
    \toprule
      & ProbDPP & Only bandit & $k$-DPP & Random \\
    \midrule
    Token-F1   & \textbf{34.14} & 22.12 & 31.20 & 26.86 $\pm$ 3.58 \\
    ROUGE-L    & \textbf{34.14} & 22.47 & 31.34 & 26.87 $\pm$ 3.61 \\
    BERTScore  & \textbf{11.27} & 8.4  & 7.67  & 4.82  $\pm$ 4.32 \\
    EM         & \textbf{25.00} & 17.00 & 24.00 & 21.00 $\pm$ 3.60 \\
    \bottomrule
  \end{tabular}
  \caption{Multi-source selection under stochastic source availability on HotpotQA (distractor setting, $N=10$, $K=3$). Higher is better; best in bold.}
  \label{tab:comparison_results2}
\end{table}

\paragraph{Multi-Source Selection under Stochastic Availability.}
Additionally, we evaluate \textsc{ProbDPP} in multi-source selection settings with stochastic source availability using the HotpotQA distractor benchmark \cite{yang2018hotpotqa}, using llm "llama3" and standard QA metrics. As shown in Table~\ref{tab:comparison_results2}, \textsc{ProbDPP} consistently outperforms all baselines on token-overlap, semantic similarity, and exact-answer metrics. In the $K=3$ setting, \textsc{ProbDPP} achieves Token-F1 and ROUGE-L scores of 34.14, substantially exceeding both the reliability-only (bandit) baseline (22.12 / 22.47) and the diversity-only (k-DPP) baseline (31.20 / 31.34). The gains are substantial for semantic and factual accuracy: \textsc{ProbDPP} attains a BERTS score of 11.27 and an exact match of 25, compared to 8.4–7.67 and 17–24 for baselines, and also exhibits lower variance than random selection $K$. ~\cref{fig:one,fig:two,fig:three,fig:four} further demonstrate robustness to stochastic availability as $\alpha$ varies. Each query is associated with $N=10$ candidate sources, from which the model selects a subset of size $K=5$, while each source independently survives with probability $\alpha$, modeling realistic partial failures. The results are averaged over 10 runs and 120 scenarios. Across all metrics, \textsc{ProbDPP} achieves the best or most stable performance, with Exact Match steadily improving from $17 \pm 5.4$ at $\alpha=0.5$ to $22.5 \pm 7.2$ at $\alpha=0.8$, while reliability-only and diversity-only baselines remain flat or degrade. Similar trends hold for ROUGE and BERTScore, indicating improved semantic alignment and preservation of relevant evidence in partial failures. The sensitivity analysis in the selection budget $K \in {3,4,5,6}$ (~\cref{fig:onek,fig:twok,fig:threek,fig:fourk}) shows that \textsc{ProbDPP} consistently outperforms or matches the baselines across all budgets, with particularly strong advantages in low-budget regimes. For example, at $K=4$, \textsc{ProbDPP} achieves EM of $21 \pm 5.2$ and ROUGE-L of $25.06 \pm 6.74$, outperforming both Highest-$\alpha$ and k-DPP, which suffer from lower accuracy and higher variance. These results indicate that \textsc{ProbDPP}’ gains are not driven by larger context sizes, but by its reliability-aware diversity objective, which selects subsets that are both informative and likely to remain available. Overall, the results confirm that jointly modeling reliability and diversity yields robust and accurate multi-source selection under constrained and stochastic conditions.

\begin{figure}[t]
  \centering

  \begin{subfigure}[t]{0.48\columnwidth}
    \centering
    \includegraphics[width=\linewidth]{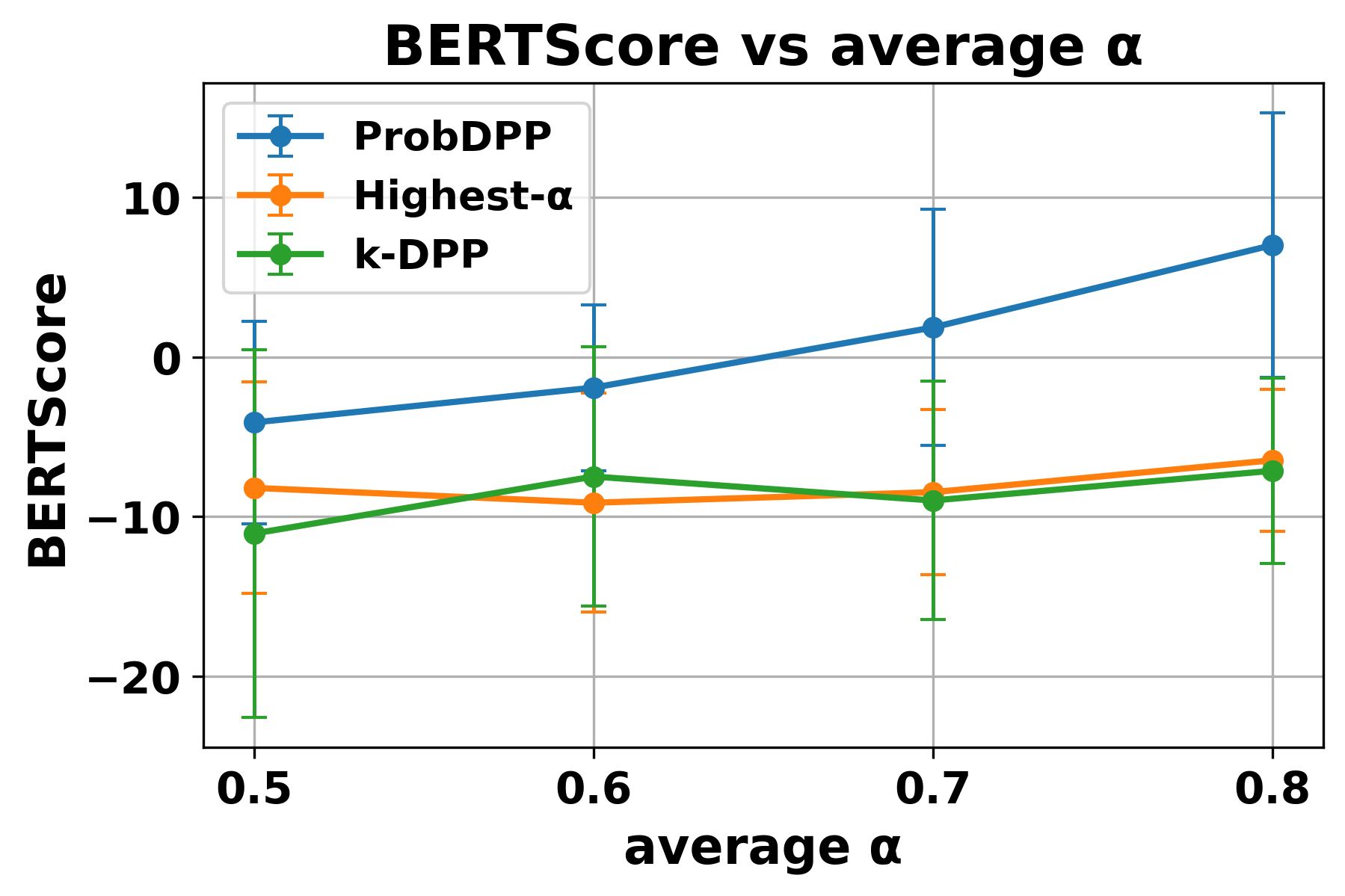}
    \caption{}
    \label{fig:one}
  \end{subfigure}\hfill
  \begin{subfigure}[t]{0.48\columnwidth}
    \centering
    \includegraphics[width=\linewidth]{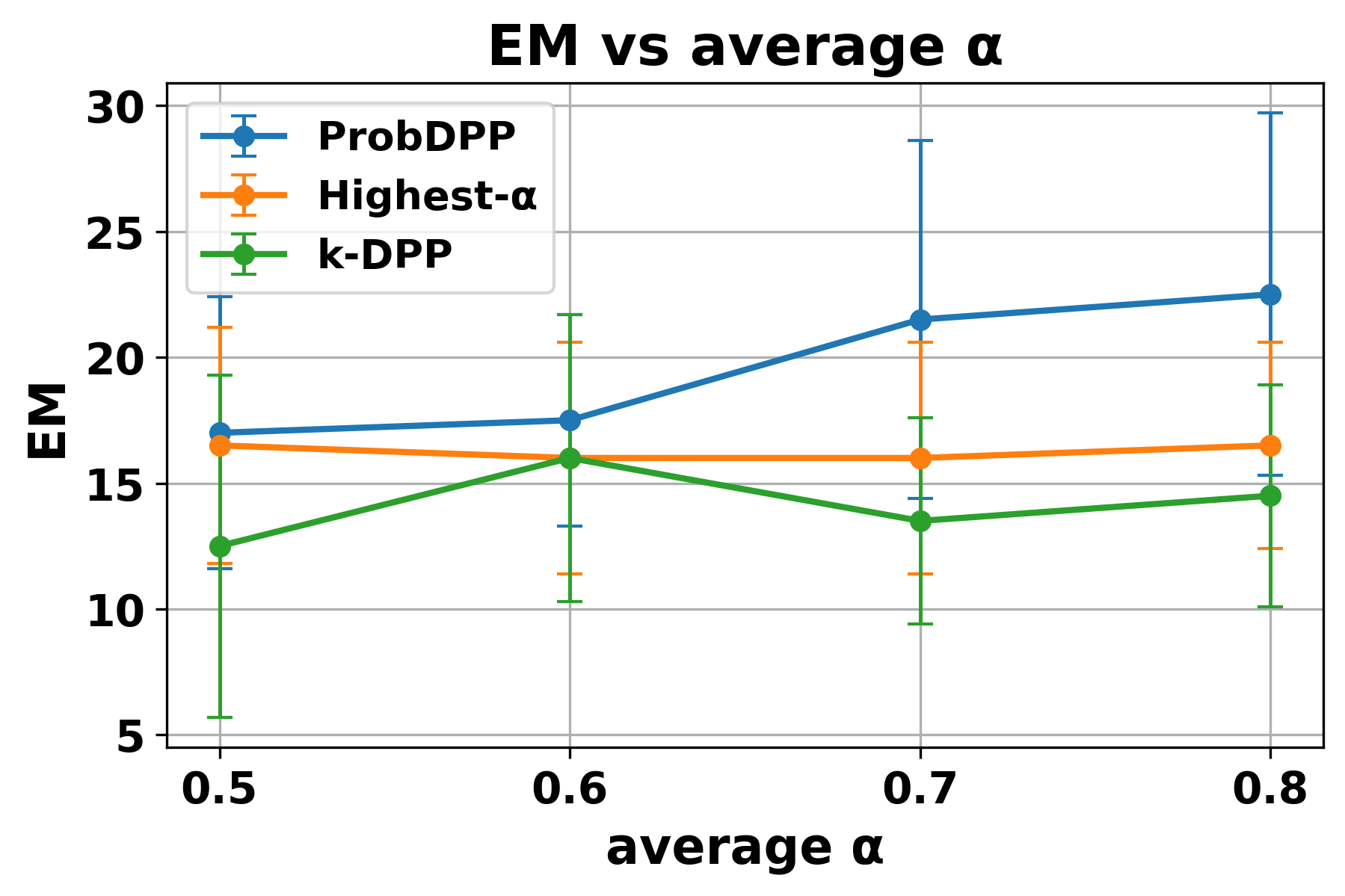}
    \caption{}
    \label{fig:two}
  \end{subfigure}

  \vspace{0.1em}

  \begin{subfigure}[t]{0.48\columnwidth}
    \centering
    \includegraphics[width=\linewidth]{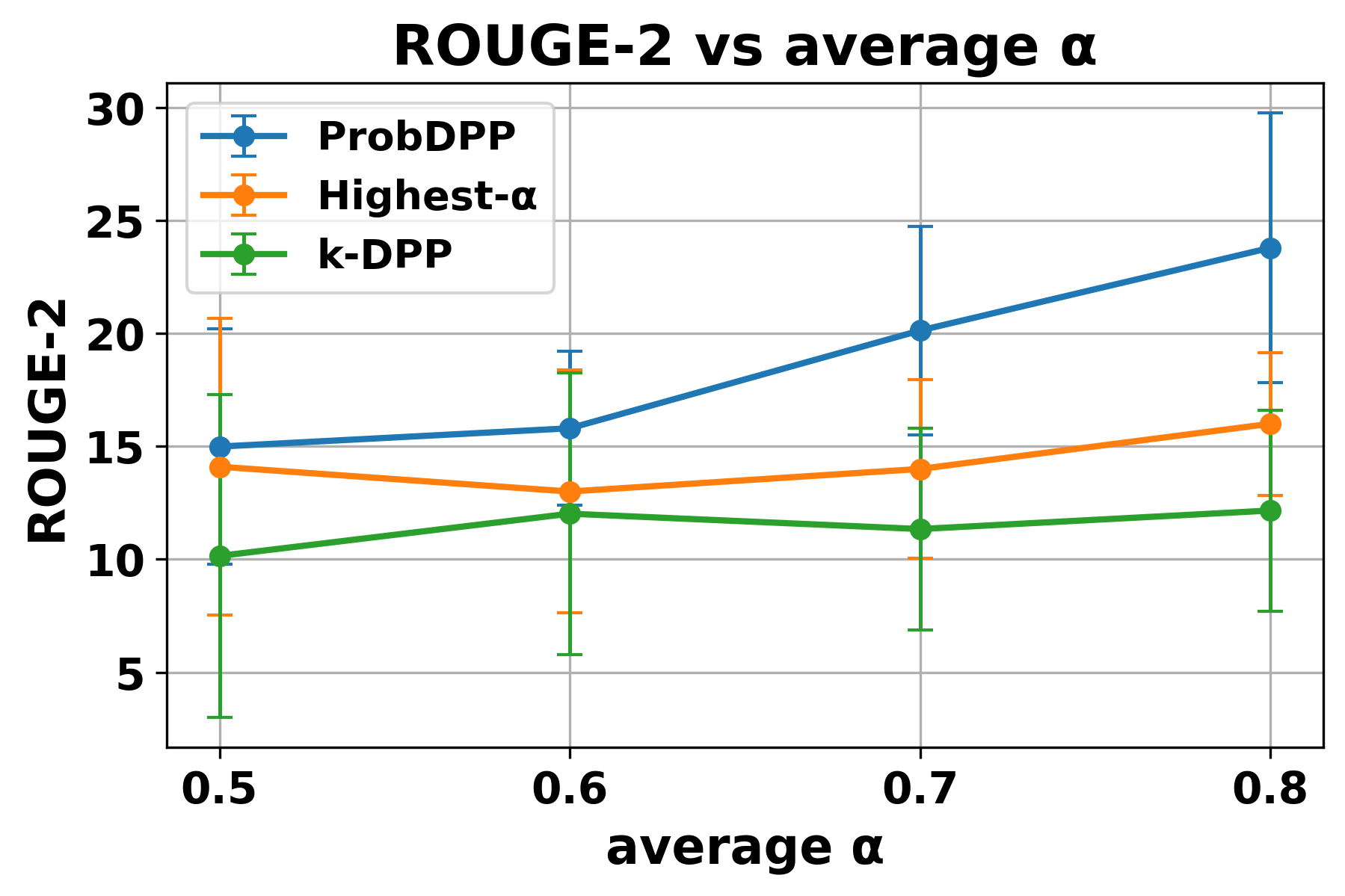}
    \caption{}
    \label{fig:three}
  \end{subfigure}\hfill
  \begin{subfigure}[t]{0.48\columnwidth}
    \centering
    \includegraphics[width=\linewidth]{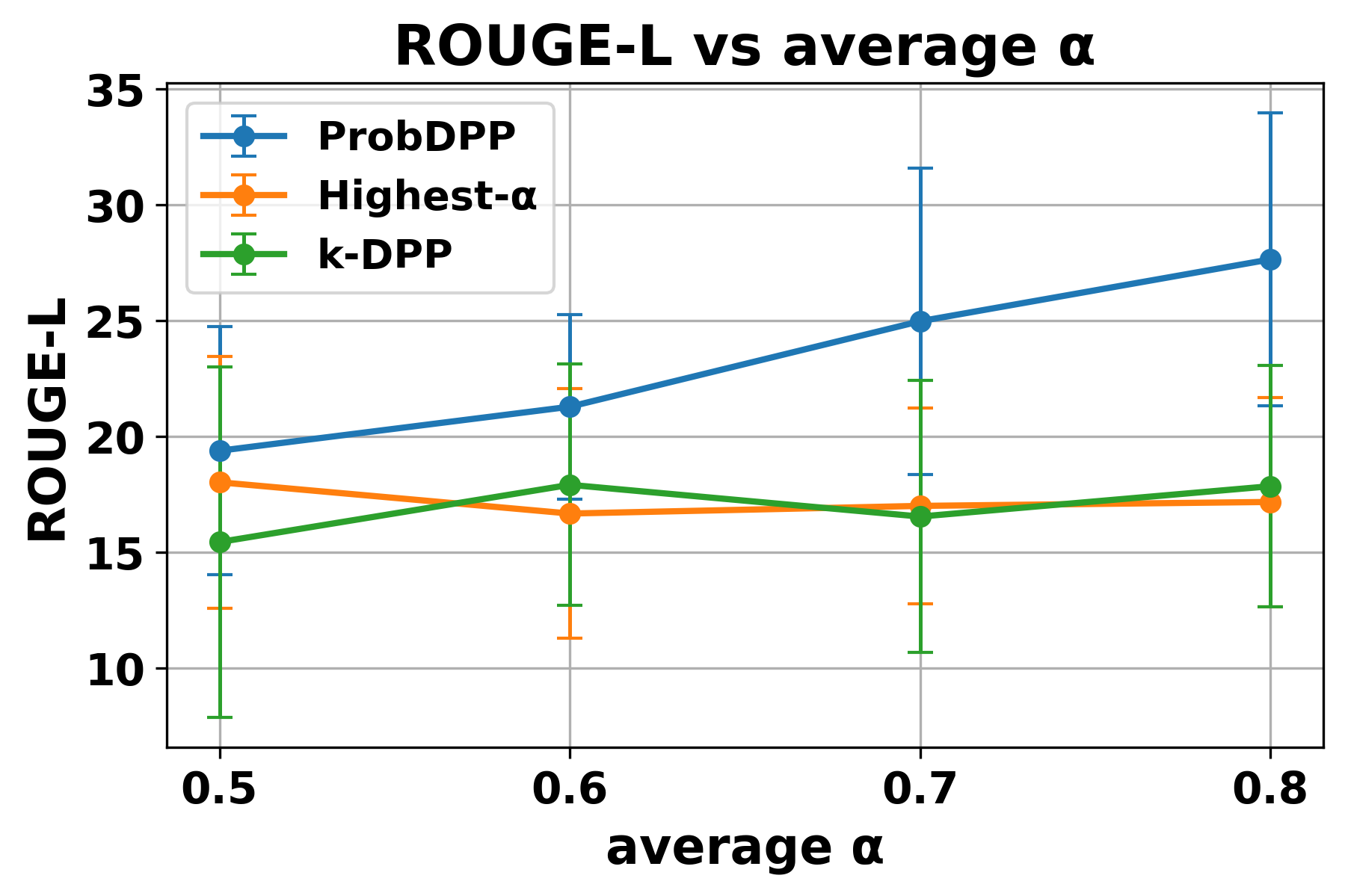}
    \caption{}
    \label{fig:four}
  \end{subfigure}

  \setlength{\abovecaptionskip}{2pt}
  \setlength{\belowcaptionskip}{-4pt}
  \caption{\small Robustness to stochastic source availability. Performance vs.\ survival average probability $\alpha$ with $N=10$ and $K=5$. Error bars show mean $\pm$ std.}
  \label{fig:robustness_alpha_grid}
\end{figure}

\begin{figure}[t]
  \centering

  \begin{subfigure}[t]{0.48\columnwidth}
    \centering
    \includegraphics[width=\linewidth]{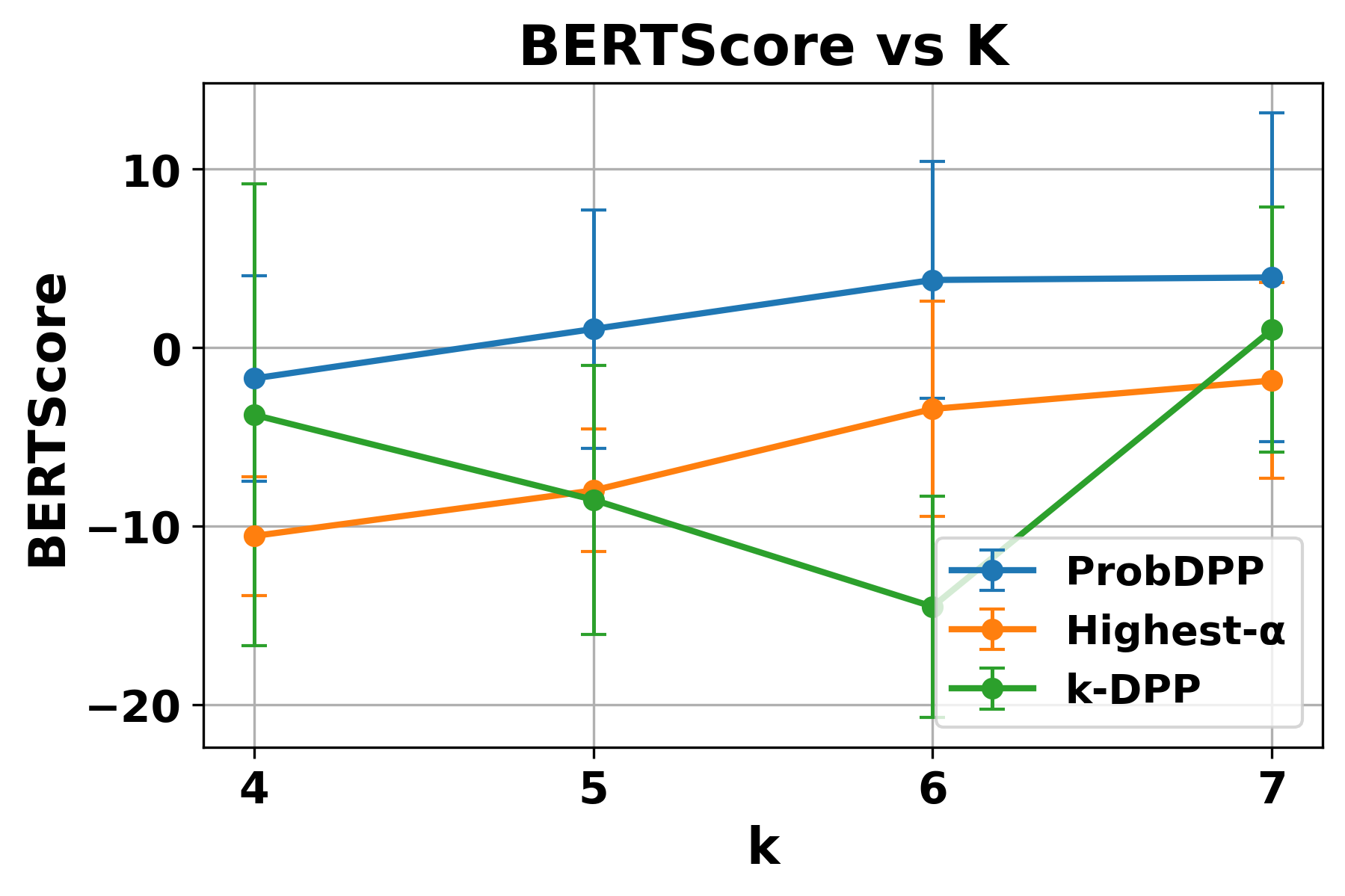}
    \caption{}
    \label{fig:onek}
  \end{subfigure}\hfill
  \begin{subfigure}[t]{0.48\columnwidth}
    \centering
    \includegraphics[width=\linewidth]{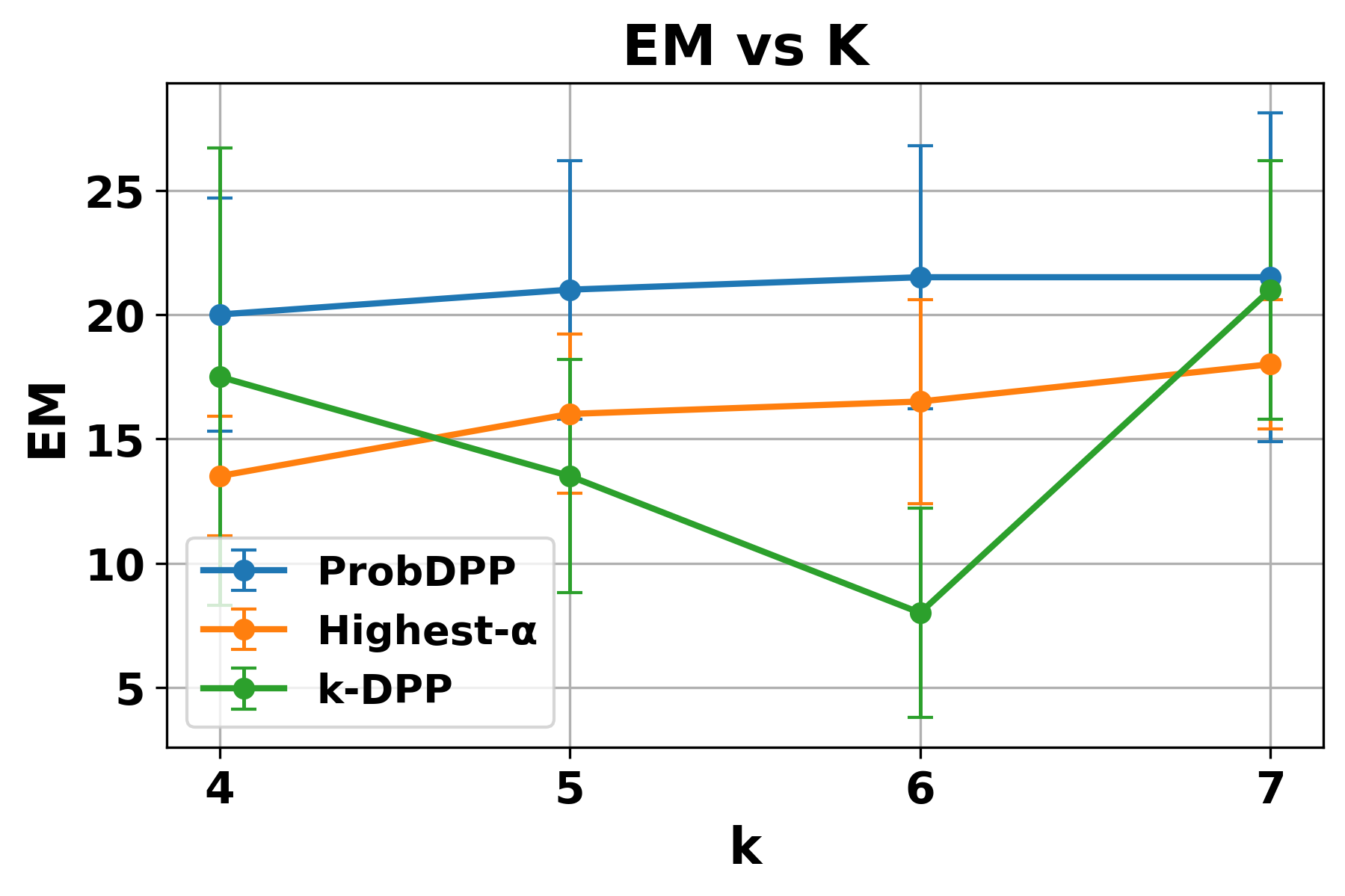}
    \caption{}
    \label{fig:twok}
  \end{subfigure}

  \vspace{0.1em}

  \begin{subfigure}[t]{0.48\columnwidth}
    \centering
    \includegraphics[width=\linewidth]{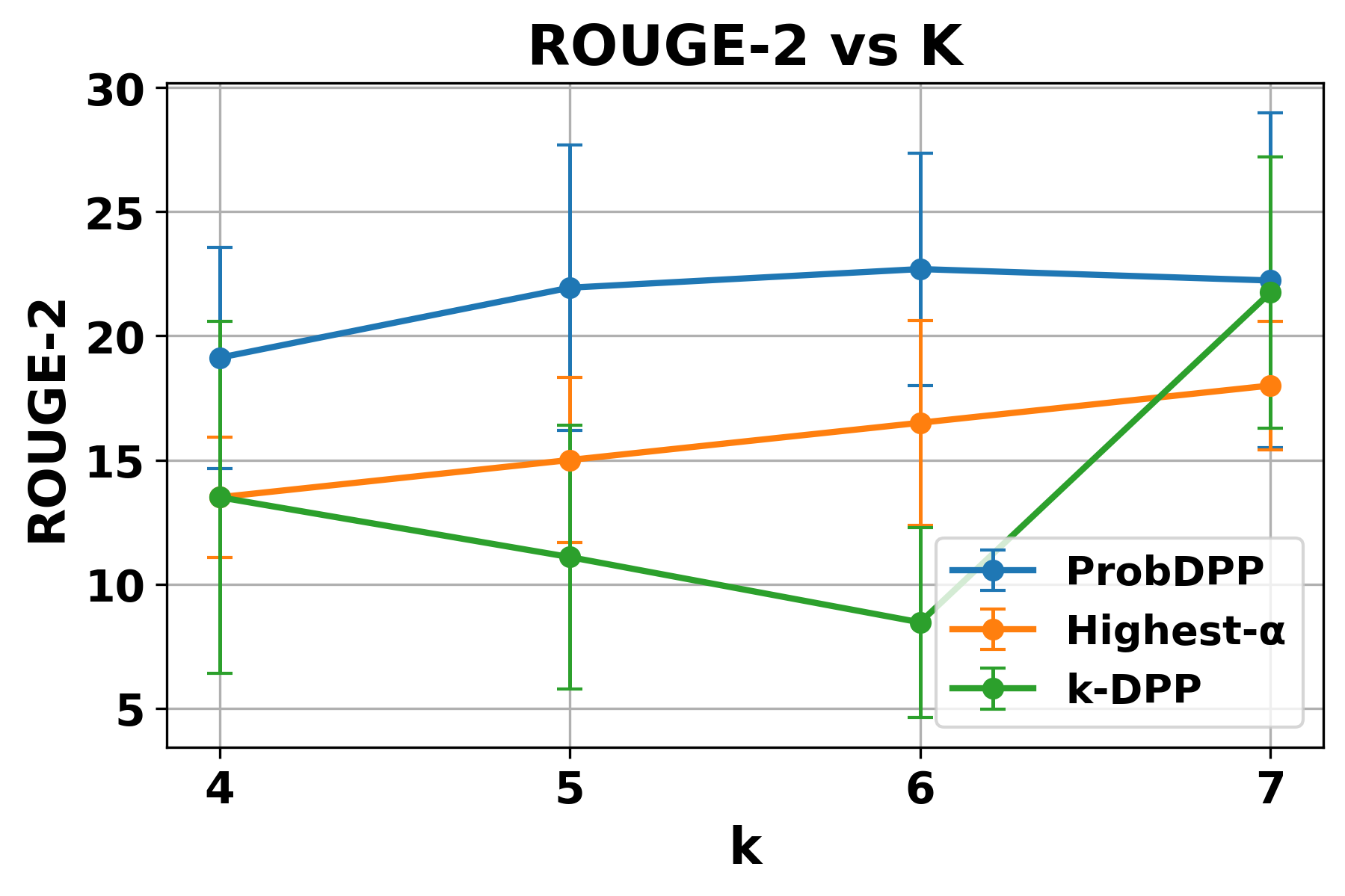}
    \caption{}
    \label{fig:threek}
  \end{subfigure}\hfill
  \begin{subfigure}[t]{0.48\columnwidth}
    \centering
    \includegraphics[width=\linewidth]{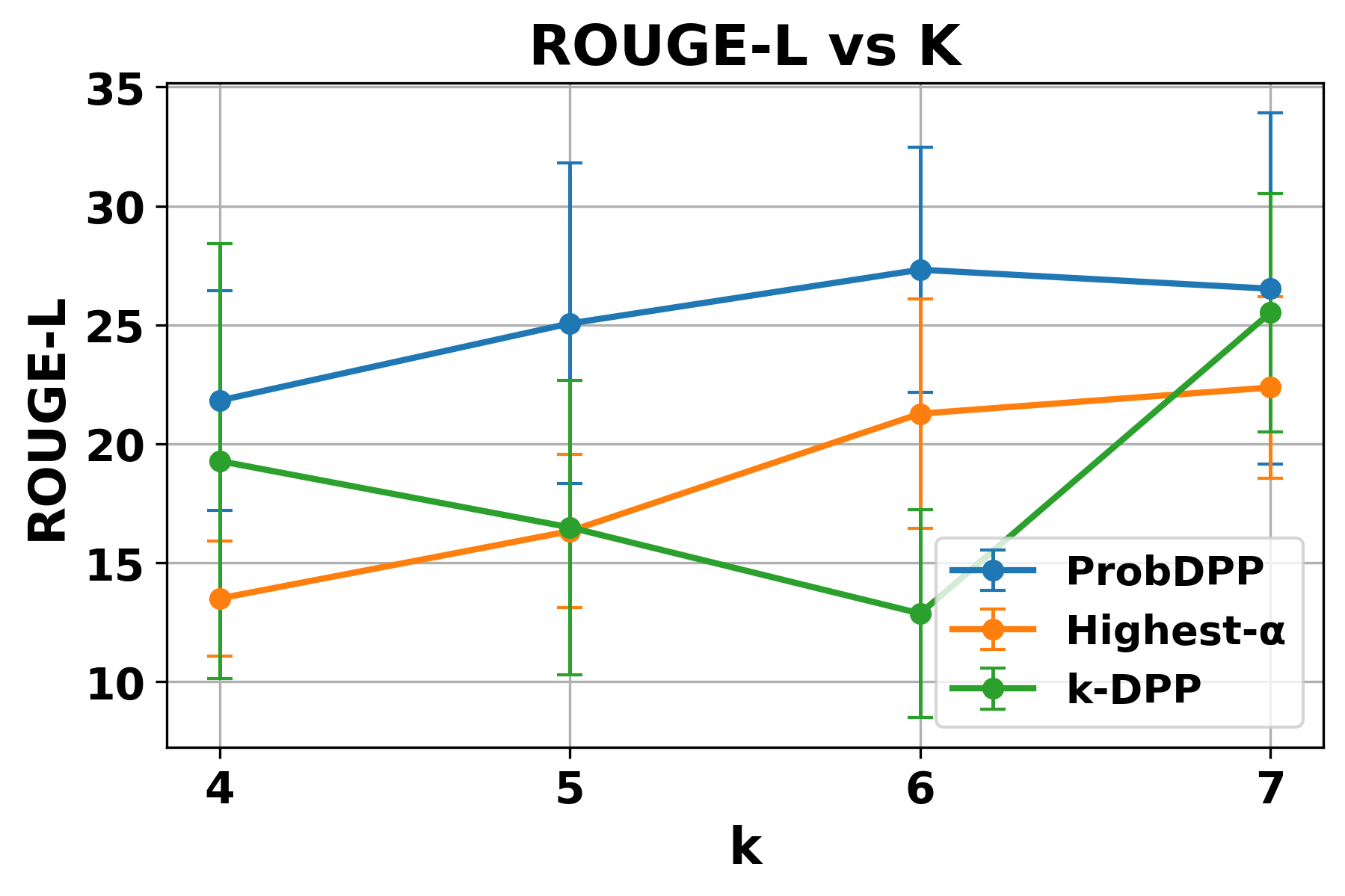}
    \caption{}
    \label{fig:fourk}
  \end{subfigure}

  \setlength{\abovecaptionskip}{2pt}
  \setlength{\belowcaptionskip}{-4pt}
  \caption{\small Effect of selection budget under stochastic availability. Performance vs.\ selection budget $K$ with $N=10$ candidate sources. Error bars show mean $\pm$ std.}
  \label{fig:budget_k_grid}
\end{figure}

\section{Discussion and Limitations}
Our analysis assumes independent data dropout, capturing random unavailability but not structured failures such as correlated or adversarial outages; extending ProbDPP to these settings is an important future direction to pursue. We further adopt a fixed similarity kernel and binary availability feedback, whereas in practice, reliability may be query-dependent and failures graded. Also, other distortion models (beyond dropout) can be considered. Our online setting assumes semi-bandit per-item feedback, while some systems observe only aggregate output quality, motivating bandit extensions for end-to-end feedback. Finally, the parameter $\varepsilon$ governs a stability–sensitivity trade-off; although a fixed value performs well empirically, adaptive selection could further improve robustness.

\section{Conclusion}
We studied budgeted unit selection for LLM pipelines enforcing diversity under stochastic data access, where selected passages may drop due to communication, storage, or RAG failures. We showed that naïve reliability-aware extensions of DPP objectives are ill-posed under Bernoulli dropouts, collapsing to $-\infty$ whenever failures occur. To address this, we proposed \textsc{ProbDPP}, a minimally regularized formulation that restores well-posedness and yields an exact decomposition into a log-det diversity term and an additive reliability reward, enabling principled optimization under probabilistic access. For unknown drop probabilities, we formulated online selection as a combinatorial semi-bandit and introduced a KL-UCB-style algorithm with matching regret bounds. Incorporating reliability directly into the diversity objective—rather than as a heuristic—consistently improves performance across stochastic-context and multi-source settings, achieving a 13.5\% BERTScore gain on MeetingBank (36.9 vs. 32.5, outperforming LLMLingua2) and a 9.4\% Token-F1 gain on HotpotQA with $K=3$ (34.14 vs. 31.20, surpassing k-DPP). Our approach provides a practical foundation for LLM fine-tuning, prompt engineering, RAG systems, and token pruning under unreliable and bandwidth-constrained communication.




\section*{Impact Statement}
This paper presents work whose goal is to advance the field of Machine Learning.
There are many potential societal consequences of our work, none which we feel must be specifically highlighted here.


\bibliography{example_paper}
\bibliographystyle{icml2026}

\newpage
\appendix
\onecolumn
\section{Appendix}

\subsection{Proof of Lemma \ref{lem:illposed}}
\label{app:illposed}

We show that the naïve expected log-determinant objective collapses under Bernoulli dropouts,
and that regularization restores well-posedness.

Let $L\subseteq[N]$ be a selected subset and define the unregularized kernel
\begin{align}
\boldsymbol{\Sigma}_{L,L}(\mathbf{z}_t) = M_L(\mathbf{z}_t)\, G_{L,L}\, M_L(\mathbf{z}_t),
\end{align}
where $z_{i,t}\sim\mathrm{Bernoulli}(\alpha_i)$ independently and $M_L(\mathbf{z}_t)=\mathrm{diag}(z_{i,t}:i\in L)$.

If any selected item $i\in L$ has $\alpha_i<1$, then
\begin{align}
\Pr(z_{i,t}=0) = 1-\alpha_i > 0,
\end{align}
and whenever $z_{i,t}=0$, the diagonal matrix $M_L(\mathbf{z}_t)$ has a zero entry.
Consequently, $\boldsymbol{\Sigma}_{L,L}(\mathbf{z}_t)$ is singular and
\begin{align}
\log\det \boldsymbol{\Sigma}_{L,L}(\mathbf{z}_t) = -\infty
\end{align}
with positive probability.
Therefore, these infinite terms contribute with a nonzero coefficient ($1-\alpha_i$) to the expected value, and we have \begin{align}
\mathbb{E}_{\mathbf{z}_t}[\log\det \boldsymbol{\Sigma}_{L,L}(\mathbf{z}_t)] = -\infty,
\end{align}
and the naïve objective is ill-posed under stochastic availability.

To resolve this issue, \textsc{ProbDPP} introduces the regularized kernel
\begin{align}
\boldsymbol{\Sigma}^{(\varepsilon)}_{L,L}(\mathbf{z}_t)
= (M_L(\mathbf{z}_t)+\varepsilon I)\, G_{L,L}\, (M_L(\mathbf{z}_t)+\varepsilon I),
\qquad \varepsilon>0.
\end{align}
Assuming $G_{L,L}\succ0$, Appendix~\ref{app:decomposition} shows that
\begin{align}
\mathbb{E}_{\mathbf{z}_t}[\log\det \boldsymbol{\Sigma}^{(\varepsilon)}_{L,L}(\mathbf{z}_t)]
=
\log\det G_{L,L}
+ \sum_{i\in L} r_i(\alpha_i,\varepsilon),
\end{align}
which is finite for all $\alpha_i\in[0,1]$.
Hence the regularized objective is well-defined and decomposes cleanly into
a geometric diversity term and an additive reliability reward.
\qed

\subsection{Proof of Lemma \ref{lem:decomp}}
\label{app:decomposition}

We derive a closed-form decomposition of the expected regularized diversity objective used in \textsc{ProbDPP}.
Fix a subset $L\subseteq[N]$ and assume that the Gram submatrix $G_{L,L}$ is positive definite.\footnote{
If $G_{L,L}$ is only positive semidefinite, the same results hold after a standard ridge $G\leftarrow G+\delta I$ with $\delta>0$.
}

The regularized kernel under stochastic availability is defined as
\begin{align}
\boldsymbol{\Sigma}^{(\varepsilon)}_{L,L}(\mathbf{z}_t)
&= W_L(\mathbf{z}_t)\, G_{L,L}\, W_L(\mathbf{z}_t), \\
W_L(\mathbf{z}_t)
&= M_L(\mathbf{z}_t) + \varepsilon I,
\end{align}
where $M_L(\mathbf{z}_t)=\mathrm{diag}(z_{i,t}:i\in L)$ and $z_{i,t}\sim\mathrm{Bernoulli}(\alpha_i)$ independently.

Since $G_{L,L}\succ 0$, it admits a Cholesky decomposition
\begin{align}
G_{L,L} = R_L R_L^\top,
\end{align}
with $R_L$ lower triangular and invertible.
Substituting into the kernel yields
\begin{align}
\boldsymbol{\Sigma}^{(\varepsilon)}_{L,L}(\mathbf{z}_t)
&= W_L(\mathbf{z}_t) R_L R_L^\top W_L(\mathbf{z}_t) \\
&= \big(W_L(\mathbf{z}_t) R_L\big)\big(W_L(\mathbf{z}_t) R_L\big)^\top.
\end{align}
Hence $\boldsymbol{\Sigma}^{(\varepsilon)}_{L,L}(\mathbf{z}_t)$ is positive definite for all $\mathbf{z}_t$ and $\varepsilon>0$.

Using the determinant identity $\det(AA^\top)=\det(A)^2$, we obtain
\begin{align}
\log\det \boldsymbol{\Sigma}^{(\varepsilon)}_{L,L}(\mathbf{z}_t)
&= 2\log\det\!\big(W_L(\mathbf{z}_t)R_L\big) \\
&= 2\log\det W_L(\mathbf{z}_t) + 2\log\det R_L \\
&= \log\det G_{L,L} + 2\sum_{i\in L}\log(z_{i,t}+\varepsilon),
\label{eq:logdet_decomp}
\end{align}
where the last equality uses $\det(R_L)^2=\det(G_{L,L})$ and the fact that $W_L(\mathbf{z}_t)$ is diagonal.

Taking expectation with respect to $\mathbf{z}_t$ gives
\begin{align}
\mathbb{E}_{\mathbf{z}_t}[\log\det \boldsymbol{\Sigma}^{(\varepsilon)}_{L,L}(\mathbf{z}_t)]
&= \log\det G_{L,L}
+ \sum_{i\in L} r_i(\alpha_i,\varepsilon),
\end{align}
where the per-item reliability reward is
\begin{align}
r_i(\alpha_i,\varepsilon)
&\triangleq
2\Big[
\alpha_i\log(1+\varepsilon)
+ (1-\alpha_i)\log\varepsilon
\Big].
\end{align}
This establishes the claimed decomposition.
\qed

\subsection{Proof of Theorem \ref{thm:upper}}
\label{app:proof_thm53}

The per-round reward is
\begin{align}
f(a) = g(a) + \langle a,\theta\rangle,
\end{align}
where $\theta_i=r_i(\alpha_i,\varepsilon)$ and $\alpha_i\in[0,1]$ is unknown.
Let $A^\star\in\arg\max_{a\in\mathcal A} f(a)$ and define the $Reg_T$
\begin{align}
Reg_T = \sum_{t=1}^T \big(f(A^\star)-f(A_t)\big).
\end{align}

From Appendix~\ref{app:decomposition},
\begin{align}
r_i(\alpha,\varepsilon)
&= 2\log\varepsilon
+ \beta_\varepsilon \alpha, \\
\beta_\varepsilon
&\triangleq
2\log\!\Big(\tfrac{1+\varepsilon}{\varepsilon}\Big) > 0,
\label{eq:affine_r}
\end{align}
so $r_i(\cdot,\varepsilon)$ is strictly increasing.

Let $U_i^\alpha(t)$ be the KL-UCB index for $\alpha_i$ and define
\begin{align}
\tilde\theta_i(t) = r_i(U_i^\alpha(t),\varepsilon).
\end{align}
Standard KL-UCB analysis ensures that with high probability,
\begin{align}
\theta_i \le \tilde\theta_i(t)
\quad \forall i,t.
\label{eq:optimism}
\end{align}

Let $\Delta(a)=f(A^\star)-f(a)$.
If $A_t$ is suboptimal, then
\begin{align}
\Delta(A_t)
&\le
\sum_{i:A_{t,i}=1}
\big(\tilde\theta_i(t-1)-\theta_i\big).
\end{align}
Hence there exists at least one selected item $i$ such that
\begin{align}
\tilde\theta_i(t-1)-\theta_i \ge \Delta_i/K,
\end{align}
where $\Delta_i$ is the minimal gap of actions containing item $i$.

Define the bad optimism indicator
\begin{align}
B_{i,t}
=
\mathbb{I}\Big\{
A_{t,i}=1,\
\tilde\theta_i(t-1)-\theta_i \ge \Delta_i/K
\Big\}.
\end{align}
Then
\begin{align}
\mathbb{I}\{A_t\neq A^\star\}
\le
\sum_{i=1}^N B_{i,t}.
\label{eq:subopt_count}
\end{align}

By \eqref{eq:affine_r}, the event
$\tilde\theta_i(t-1)-\theta_i \ge \Delta_i/K$
is equivalent to
\begin{align}
U_i^\alpha(t-1) \ge \alpha_i^+
\triangleq
\min\!\Big\{
1,\
\alpha_i + \tfrac{\Delta_i}{K\beta_\varepsilon}
\Big\}.
\end{align}
KL-UCB counting arguments yield
\begin{align}
\mathbb{E}\!\left[\sum_{t=1}^T B_{i,t}\right]
\le
\frac{\log T + c\log\log T}{\mathrm{kl}(\alpha_i\|\alpha_i^+)}
+ O(1).
\end{align}

Let
\begin{align}
\Delta_{\max} = \max_{a\in\mathcal A}\big(f(A^\star)-f(a)\big),
\end{align}
which is finite since $g(a)$ and $r_i(\cdot,\varepsilon)$ are bounded.
Using \eqref{eq:subopt_count},
\begin{align}
\mathbb{E}[Reg_T]
&\le
\Delta_{\max}\sum_{i=1}^N
\mathbb{E}\!\left[\sum_{t=1}^T B_{i,t}\right] \\
&\le
\Delta_{\max}\sum_{i=1}^N
\frac{\log T + c\log\log T}{\mathrm{kl}(\alpha_i\|\alpha_i^+)}
+ O(1).
\end{align}
This proves the claimed logarithmic regret bound.
\qed

\subsection{Lower Bound for ProbDPP}
\label{app:proof_lower}

We use the restricted instance from Section~5.4 with $K=1$ and $g(a)$ constant over singleton sets, so each action selects a single item $i\in[N]$ and the learner observes Bernoulli feedback $z_{i,t}\sim Bernoulli(\alpha_i)$.
Moreover, the reliability reward in Eq.~(53) is affine in $\alpha_i$:
\[
r_i(\alpha_i,\varepsilon)
= 2\log\varepsilon \;+\; \beta_\varepsilon\,\alpha_i,
\qquad
\beta_\varepsilon \triangleq 2\log\!\Big(\frac{1+\varepsilon}{\varepsilon}\Big).
\]
Hence maximizing the expected objective is equivalent (up to an additive constant and scaling by $\beta_\varepsilon$) to maximizing $\alpha_i$.
Let $i^\star\in\arg\max_i \alpha_i$ and define the (objective) gap
$\Delta_i \triangleq \beta_\varepsilon(\alpha_{i^\star}-\alpha_i)$.

For any consistent policy $\pi$ on the unstructured class of Bernoulli bandits, Theorem~16.2 in \citet{lattimore2020bandit}
(Burnetas--Katehakis / Lai--Robbins form) gives, for each suboptimal arm $i\neq i^\star$,

\begin{align}
\liminf_{T\to\infty}
\frac{\mathbb{E}_\pi\!\left[n_i(T)\right]}{\log T}
&\ge
\frac{1}{d_i},
\\
d_i
&\triangleq
\mathrm{kl}\!\big(\alpha_i \,\|\, \alpha_{i^\star}\big).
\end{align}

where $\mathcal{B}$ denotes the Bernoulli family. By Table~16.1 in \citet{lattimore2020bandit} (Bernoulli specialization),
$d_i = \mathrm{kl}(\alpha_i\|\alpha_{i^\star})$, where
$\mathrm{kl}(p\|q)=p\log\frac{p}{q}+(1-p)\log\frac{1-p}{1-q}$.
Therefore,
\[
\liminf_{T\to\infty}\frac{Reg_T}{\log T}
=
\liminf_{T\to\infty}\sum_{i\neq i^\star}\Delta_i\frac{\mathbb{E}_\pi[n_i(T)]}{\log T}
\;\ge\;
\sum_{i\neq i^\star}\frac{\Delta_i}{\mathrm{kl}(\alpha_i\|\alpha_{i^\star})}.
\]
This is the standard instance-dependent lower bound with information denominators, specialized to the Bernoulli-reliability setting induced by the $K=1$ reduction.


\end{document}